\documentclass{article}
\usepackage{ijcai19}

\usepackage{times}
\usepackage{soul}
\usepackage{url}
\usepackage[hidelinks]{hyperref}
\usepackage[utf8]{inputenc}
\usepackage[small]{caption}
\usepackage{amsmath}
\usepackage{stackengine,booktabs,subcaption}
\usepackage{amsfonts}
\usepackage{amsthm}
\usepackage{amsmath}
\usepackage{natbib}
\usepackage{xspace}
\usepackage{color}
\usepackage{algorithm}
\usepackage{algorithmicx}
\usepackage{graphicx}
\usepackage[noend]{algpseudocode}
\algdef{SE}[SUBALG]{Indent}{EndIndent}{}{\algorithmicend\ }%
\algtext*{Indent}
\algtext*{EndIndent}
\title{On Constrained Open-World Probabilistic Databases}
\author{Tal Friedman \And Guy Van den Broeck \\
  \affiliations
  Department of Computer Science \\ University of California, Los Angeles \\
  \emails
  \{tal, guyvdb\}@cs.ucla.edu
}

\newtheorem{theorem}{Theorem}
\newtheorem{proposition}[theorem]{Proposition}
\newtheorem{lemma}[theorem]{Lemma}
\newtheorem{corollary}[theorem]{Corollary}
\theoremstyle{definition}
\newtheorem{definition}{Definition}

\newcommand{\gf}{\mathcal{G}}
\newcommand{\pdb}{\mathcal{P}}
\newcommand{\rf}{\mathcal{R}}

\newcommand{\tf}{\mathcal{T}}
\newcommand{\df}{\mathcal{D}}
\newcommand{\kgp}{K_\gf^\Phi}
\newcommand{\otups}{\mathbf{O}}
\newcommand{\Query}{\ensuremath{\mathsf{Q}}\xspace}

\newcommand{\ax}[1]{\ensuremath{\left<#1\right>}\xspace}

\begin{document}

\maketitle

\begin{abstract}
    Increasing amounts of available data have led to a heightened need for representing large-scale probabilistic knowledge bases. One approach is to use a probabilistic database, a model with strong assumptions that allow for efficiently answering many interesting queries. Recent work on open-world probabilistic databases strengthens the semantics of these probabilistic databases by discarding the assumption that any information not present in the data must be false. While intuitive, these semantics are not sufficiently precise to give reasonable answers to queries. We propose overcoming these issues by using constraints to restrict this open world. We provide an algorithm for one class of queries, and establish a basic hardness result for another. Finally, we propose an efficient and tight approximation for a large class of queries.

\end{abstract}

\section{Introduction}
An ubiquitous pursuit in the study of knowledge base representation is the search for a model that can represent uncertainty while simultaneously answering interesting queries efficiently. The key underlying challenge is that these goals are at odds with each other. Modelling \emph{uncertainty} requires additional model complexity. At the same time, the ability to answer \emph{meaningful queries} usually demands fewer model assumptions. Both of these properties are at odds with the key limiting factor of \emph{tractability}: success in the first two goals is not nearly as impactful if it is not achieved efficiently. Unfortunately, probabilistic reasoning is often computationally hard, even on databases~\citep{roth1996hardness,Dalvi2012TheDO}.

One approach towards achieving this goal is to begin with a simple model such a probabilistic database (PDB) \citep{Suciu:2011:PD:2031527,VdBFTDB17}. A PDB models uncertainty, but is inherently simple and makes very strong independence assumptions and closed-world assumptions allowing for tractability on a very large class of queries \citep{Dalvi2007TheDO,Dalvi2012TheDO}. However, PDBs can fall short under non-ideal circumstances, as their semantics are brittle to incomplete knowledge bases \citep{Ceylan2016OpenWorldPD}.

To bring PDBs closer to the desired goal, \citet{Ceylan2016OpenWorldPD} propose open-world probabilistic databases (OpenPDB), wherein the semantics of a PDB are strengthened to relax the closed-world assumption. While OpenPDBs maintain a large class of tractable queries, their semantics are so relaxed these queries lose their precision: they model further uncertainty, but in exchange give less useful query answers.

In this work, we aim to overcome these querying challenges, while simultaneously maintaining the degree of uncertainty modeled by OpenPDBs. To achieve this, we propose further strengthening the semantics of OpenPDBs by constraining the mean probability allowed for a relation. These constraints work at the \emph{schematic} level, meaning no additional per-item information is required. They are practically motivated by knowledge of summary statistics, of how many tuples we expect to be true. A theoretical analysis shows that, despite their simplicity, such constraints fundamentally change the difficulty landscape of queries, leading us to propose a general-purpose approximation scheme.

The rest of the paper is organized as follows: Section~\ref{sec:background} provides necessary background on relational logic and PDBs, as well as an introduction to OpenPDBs. Section~\ref{sec:mtp} motivates and introduces our construction for constraining OpenPDBs. Section~\ref{sec:exact} analyses exact solutions subject to these constraints, providing a class of tractable queries along with an algorithm. It also shows that the problem is in general hard, even in some cases where standard PDB queries are tractable. Section~\ref{sec:approx} investigates an efficient and provably bounded approximation scheme. Section~\ref{sec:disc} discusses our findings, and summarizes interesting directions that we leave as open problems.

\section{Background} \label{sec:background}

This section provides background and motivation for probabilistic databases and their open-world counterparts. 
Notation and definitions are adapted from \citet{Ceylan2016OpenWorldPD}.

\subsection{Relational Logic and Databases}
We now describe necessary background from \emph{function-free finite-domain} first-order logic. An atom $R(x_1,x_2,...,x_n)$ consists of a predicate $R$ of arity $n$, together with $n$ arguments. These arguments can either be \emph{constants} or \emph{variables}. A \emph{ground atom} is an atom that contains no variables. A \emph{formula} is a series of atoms combined with conjunctions ($\land$) or disjunctions ($\lor$), and with quantifiers $\forall, \exists$. A \emph{substitution} $\Query[x/t]$ replaces all occurences of $x$ by $t$ in a formula $\Query$.

A relational \emph{vocabulary} $\sigma$ is comprised of a set of predicates $\rf$ and a domain $\df$. Using the \emph{Herbrand semantics} \citep{hinrichs2006herbrand}, the \emph{Herbrand base} of $\sigma$ is the set of all ground atoms possible given $\rf$ and $\df$. A $\sigma$-interpretation $\omega$ is then an assignment of truth values to every element of the Herbrand base of $\sigma$. We say that $\omega$ \emph{models} a formula $\Query$ whenever $\omega$ satisfies $\Query$. This is denoted by $\omega \models \Query$.

\begin{figure}
\centering
\scalebox{0.8}{
\setlength{\tabcolsep}{3pt}
\begin{tabular}{l}
\toprule
 \multicolumn{1}{c}{Scientist} \\  
\midrule
Einstein \\
Erd\H{o}s \\
von Neumann \\
\bottomrule
\end{tabular}
}
\qquad
\scalebox{0.8}{
\setlength{\tabcolsep}{3pt}
\begin{tabular}{l l}
\toprule
 \multicolumn{2}{c}{CoAuthor}  \\
\midrule
 Einstein & Erd\H{o}s \\
 Erd\H{o}s & von Neumann \\
\bottomrule
\end{tabular}
}
\caption{Example relational database. Notice that the first row of the right table corresponds to the atom CoAuthor(Einstein, Erd\H{o}s). \label{fig:exdb}}
\end{figure}

A reasonable starting point for the target knowledge base to construct would be to use a traditional \emph{relational database}. Using the standard model-theoretic view \citep{Abiteboul1995FoundationsOD}, a relational database for a vocabulary $\sigma$ is a $\sigma$-interpretation $\omega$. Less formally, a relational database consists of a series of relations, each of which corresponds to a predicate. Each relation consists of a series of rows, also called \emph{tuples}, each of which corresponds to an atom of the predicate being true. Any atom not appearing as a row in the relation is considered to be \textit{false}, following the closed-world assumption \citep{reiter1981closed}. Figure~\ref{fig:exdb} shows an example database. 

\subsection{Probabilistic Databases}

\begin{figure}
\centering
\scalebox{0.8}{
\setlength{\tabcolsep}{3pt}
\begin{tabular}{l | c}
\toprule
 \multicolumn{1}{c|}{Scientist} & $\Pr$   \\
\midrule
Einstein & 0.8 \\
Erd\H{o}s & 0.8 \\
von Neumann & 0.9 \\
Shakespeare & 0.2 \\
\bottomrule
\end{tabular}
}
\qquad
\scalebox{0.8}{
\setlength{\tabcolsep}{3pt}
\begin{tabular}{l l | c}
\toprule
 \multicolumn{2}{c|}{CoAuthor} & $\Pr$ \\
\midrule
 Einstein & Erd\H{o}s & 0.8 \\
 Erd\H{o}s & von Neumann & 0.9 \\
 von Neumann & Einstein & 0.5 \\
\bottomrule
\end{tabular}
}
\caption{Example probabilistic database. Tuples are now of the form $\langle t:p \rangle$ where $p$ is the probability of the tuple $t$ being present. \label{fig:expdb}}
\end{figure}

Despite the success of relational databases, their deterministic nature leads to a few shortcomings. A common way to gather a large knowledge base is to apply some sort of statistical model \citep{Carlson2010TowardAA,yago,Peters2014AMR, Dong14Vault} which returns a probability value for potential tuples. Adapting the output of such a model to a relational database involves thresholding on the probability value, discarding valuable information along the way. A \emph{probabilistic database} (PDB) circumvents this problem by assigning each tuple a probability.

\begin{definition}
 A \emph{(tuple-independent) probabilistic database} $\pdb$ for a vocabulary $\sigma$ is a finite set of tuples of the form $\langle{t: p}\rangle$ where $t$ is a $\sigma$-atom and $p\in [0,1]$. Furthermore, each $t$ can appear at most once. 
\end{definition}

Given such a collection of tuples and their probabilities, we are now going to define a \emph{distribution} over relational databases. The semantics of this distribution are given by treating each tuple as an independent random variable.

\begin{definition}
  A probabilistic database $\pdb$ for vocabulary $\sigma$ induces a probability distribution over $\sigma$-interpretations $\omega$:
  \begin{align*}
     & P_\pdb(\omega) = \prod_{t \in \omega} P_\pdb(t) \prod_{t \notin \omega} (1-P_\pdb(t))\\
     & \text{ where~~~}  P_\pdb(t) = 
      \begin{cases}
      p & \textnormal{if } \langle t:p \rangle \in \pdb \\
      0 & \textnormal{otherwise}
      \end{cases}
  \end{align*}
\end{definition}

Notice this last statement is again making the closed-world assumption: any tuple that we have no information about is assigned probability zero. Figure~\ref{fig:expdb} shows an example PDB.

\paragraph{Probabilistic Queries}
In relational databases, the fundamental task we are interested in solving is how to answer queries. The same is true for probabilistic databases, with the only difference being that we are now interested in probabilities over queries. In particular, we are interested in queries that are fully quantified - also known as \emph{Boolean queries}. On a relational database, this corresponds to a query that has an answer of True or False.

For example, on the database given in Figure~\ref{fig:exdb}, we might ask if there is a scientist who is a coauthor:
$$ \Query_1 = \exists x. \exists y. S(x) \land \mathit{CoA}(x,y)$$

If we instead asked this query of the probabilistic database in Figure~\ref{fig:expdb}, we would be computing the probability by summing over the worlds in which the query is true:
$$ P(\Query_1) = \sum_{\omega \models \Query_1} P_\pdb(\omega) $$

Queries of this form that are a conjunction of atoms are called \emph{conjunctive queries}. They are commonly shortened as:
$$ \Query_1 = S(x), \mathit{CoA}(x,y).$$

A disjunction of conjunctive queries is known as a \emph{union of conjunctive queries} (UCQ). UCQs have been shown to live in a dichotomy of efficient evaluation \citep{Dalvi2012TheDO}: computing the probability of a UCQ is either polynomial in the size of the database, or it is $\#P$-hard. This property can be checked through the syntax of a query, and we say that a UCQ is \emph{safe} if it admits efficient evaluation. In the literature of probabilistic databases \citep{Suciu:2011:PD:2031527,Dalvi2012TheDO}, as well as throughout the rest of this paper, UCQs are the primary query object studied.


\subsection{Open-World Probabilistic Databases}
In the context of automatically constructing a knowledge base, as is done in for example NELL \citep{Carlson2010TowardAA} or Google's Knowledge Vault \citep{Dong14Vault}, making the closed-world assumption is conceptually unreasonable. Conversely, it is also not feasible to include all possible tuples and their probabilities in the knowledge base. The resulting difficulty is that there are an enormous number of probabilistic facts that can be scraped from the internet, and by definition these tools will keep only those with the very highest probability. As a result, knowledge bases like NELL \citep{Carlson2010TowardAA}, PaleoDeepDive \citep{Peters2014AMR}, and YAGO \citep{yago} consist almost entirely of probabilities above $0.95$. This tells us that the knowledge base we are looking at is fundamentally \emph{incomplete}. In response to this problem, \citet{Ceylan2016OpenWorldPD} propose the notion of a \emph{completion} for a probabilistic database.

\begin{definition}
A \emph{$\lambda$-completion} of a probabilistic database $\pdb$ is another
probabilistic database obtained as follows. For each atom $t$ that does not
appear in $\pdb$, we add tuple $\langle{t:p}  \rangle$ to $\pdb$ for some $p \in [0,\lambda]$.
\end{definition}

Then, we can define the open world of possible databases in terms of the set of distributions induced by all completions.
\begin{definition}
  An \emph{open-world probabilistic database} (OpenPDB) is a pair $\gf=(\pdb, \lambda)$,
  where $\pdb$ is a probabilistic database and $\lambda \in [0,1]$. $\gf$ induces a set of probability distributions
 $K_{\gf}$ such that a distribution P belongs to $K_{\gf}$ iff P is induced by
 some $\lambda$-completion of probabilistic database $\pdb$. 
\end{definition}

\paragraph{Open-World Queries}
OpenPDBs specify a set of probability distributions rather than a single one, meaning that a given query produces a set of possible probabilities rather than a single one. We focus on computing the minimum and maximum possible probability values that can be achieved by completing the database. 
\begin{definition}
  The \emph{probability interval of a Boolean query $\Query$} in OpenPDB $\gf$ is
  $K_{\gf}(\Query)=[\underline{P}_{\gf}(\Query), \overline{P}_{\gf}(\Query)]$, where
  \begin{align*}
  \underline{P}_{\gf}(\Query)=\min_{P\in K_{\gf}}P(\Query) && \overline{P}_{\gf}(\Query)=\max_{P\in K_{\gf}}P(\Query)] 
\end{align*}
\end{definition}

In general, computing the probability interval for some first-order $\Query$ is not
tractable. As observed in \citet{Ceylan2016OpenWorldPD}, however, the situation is different for UCQ queries, because they are monotone (they contain no negations). For UCQs, the upper and lower bounds are given respectively by the full completion
(where all unknown probabilities are $\lambda$), and the closed world
database. This is a direct result of the fact that OpenPDBs form a credal set: a closed convex set of probability measures, meaning that probability bounds always come from extreme points~\citep{Cozman2000CredalN}.
Furthermore, \citet{Ceylan2016OpenWorldPD} also provide an algorithm for efficiently computing this upper bound corresponding to a full completion, and show that it works whenever the UCQ is safe.

\section{Mean-Constrained Completions}\label{sec:mtp}

This section motivates the need to strengthen the OpenPDB semantics, and introduces our novel probabilistic data model.

\subsection{Motivation}
The ability to perform efficient query evaluation provides an appealing case for OpenPDBs.
They give a more reasonable semantics, better matching their use, and for a large class of queries they come at no extra cost in comparison to traditional PDBs. 
However, in practice computing an upper bound in this
way tends to give results very close to $1$. Intuitively, this makes sense: our upper bound comes from simultaneously assuming that \emph{every} possible missing atom has some reasonable probability. While such a bound is easy to compute, it is too strong of a relaxation of the closed-world assumption.

The crux of this issue is that OpenPDBs consider every possible circumstance for unknown tuples: even ones that are clearly unreasonable. For example, suppose that a table in our database describes whether or not a person is a scientist. The OpenPDB model considers the possibility that every person it knows nothing about has a nontrivial probability of being a scientist - this will clearly return nonsensical query results as we know that fewer than 1\% of the population are scientists.

In order to consider a restricted subset of completions representing reasonable situations, we propose directly incorporating these summary statistics. Specifically, we place constraints on the overall probability of a relation across the entire population. In the scientist example, our model only considers completions in which the total probability mass of people being scientists totals less than 1\%. This allows us to include more information at the schema level, without having more information about each individual.



To illustrate the effect this has, consider a schema in which we have 3 relations: $\mathit{LiLA}(x)$ denoting whether one lives in Los Angeles, $\mathit{LiSpr}(x)$ denoting whether one lives in Springfield, and $S(x)$ denoting whether one is a scientist. Using a vocabulary of 500 people where each person is present in at most one relation, Table~\ref{tab:probcomparison} shows the resulting upper probability bound under different model assumptions, where the constrained open-world restricts at most $50\%$ of mass on $\mathit{LiLA}$, $5\%$ on $S$, and $0.5\%$ on $\mathit{LiSpr}$. In particular, notice how extreme the difference is in upper bound with and without constraints being imposed. The closed-world probability of both of these queries is always 0, as each person in our database only has a known probability for at most one relation. It is clear that of these three options, the constrained open-world is the most reasonable -- the rest of this section formalizes this idea and investigates the resulting properties.

\begin{table}[tb]
\centering
\setlength{\tabcolsep}{3pt}
\begin{tabular}{p{80pt}p{25pt}p{60pt}p{40pt}}
    \toprule
    Query & CW & OW & COW \\
    \midrule 
    $\mathit{LiLA}(x), S(x)$ & $0$ & $1-10^{-290}$ & $1-10^{-15}$ \\ 
    $\mathit{LiSpr}(x), S(x)$ & $0$ & $1-10^{-191}$ & $0.96$\\
    \bottomrule
\end{tabular}
\caption{Comparison of upper bounds for the same query and database with different model assumptions: Closed-World (CW), Open-World (OW), and Constrained Open-World (COW). }
\label{tab:probcomparison}
\end{table}

\subsection{Formalization}
We begin here by defining mean based constraints, before examining some
immediate observations about the structure of the resulting constrained
database.

\begin{definition}
  \label{def:mtp}
Suppose we have a PDB $\pdb$, and let $\mathit{Tup}(R) \subseteq \pdb$ be the set of probabilistic tuples in relation $R$. Let $\bar{p}$ be a probability threshold. 
Then a \emph{mean tuple probability constraint} (MTP constraint) $\varphi$ is a linear constraint of the form
\begin{equation*}
  \bar{p} > \frac{1}{|\mathit{Tup}(R)|} \sum_{\langle t:p \rangle \in \mathit{Tup}(R)} p
\end{equation*}
\end{definition}


\begin{definition}
  We say that a $\lambda$-completion is \emph{$\varphi$-constrained} if the
  $\lambda$-completed database satisfies MTP $\varphi$. If it satisfies all of
  $\Phi=(\varphi_1, \varphi_2, ..., \varphi_n)$, then we say it is $\Phi$-constrained.
\end{definition}

Being $\varphi$-constrained is not a property of OpenPDBs, but of their PDB completions. Hence, we are interested in the subset of completions
that satisfy this property.

\begin{definition}
 An OpenPDB $\gf=(\pdb, \lambda)$ together with MTP constraints $\Phi$ induces a
 set of probability distributions $K_{\gf}^\Phi$ , where distribution P belongs
 to $K_{\gf}^\Phi$ iff P is induced by some $\Phi$-constrained
 $\lambda$-completion of $\pdb$.
\end{definition}

Much like with standard OpenPDBs, for a Boolean query $\Query$ we are
interested in computing bounds on $P(\Query)$.

\begin{definition}
 The probability interval of a Boolean query $\Query$ in OpenPDB $\gf$ with MTP
 constraints $\Phi $ is $\kgp(\Query)=[\underline{P}_{\gf}^\Phi(\Query),
 \overline{P}_{\gf}^\Phi(\Query)]$, where
 \begin{align*}
  \underline{P}_{\gf}^\Phi(\Query)=\min_{P\in \kgp}P(\Query) && \overline{P}_{\gf}^\Phi(\Query)=\max_{P\in \kgp}P(\Query)] 
\end{align*}

\end{definition}

\subsection{Completion Properties}

A necessary property of OpenPDBs for efficient query evaluation is that they are credal -- this is what allows us to consider only a finite subset of possible completions. MTP-constrained OpenPDBs maintain this property.\footnote{Proofs of all theorems and lemmas are given in the appendix}

\begin{proposition}
  \label{thm:credal}
  Suppose we have an OpenPDB $\gf$ together with MTP constraints $\Phi$. Then
  the induced set of probability distributions $\kgp$ is credal.
\end{proposition}

This property allows us to examine only a finite subset of configurations when
looking at potential completions, since query probability bounds of a credal set are always achieved at points of extrema \citep{Cozman2000CredalN}. Next, we would like to characterize these points
of extrema, by showing that the number of tuples not on their own individual
boundaries (that is, $0$ or $\lambda$) is given by the number of MTP constraints.

\begin{theorem}
  \label{thm:extrema}
    Suppose we have an OpenPDB $\gf=(\pdb, \lambda)$ with MTP constraints $\Phi$, and a
  UCQ $Q$. If $\pdb'$ is a $\Phi$-constrained $\lambda$-completion satisfying $\kgp(\Query)=[P_\pdb(\Query),P_{\pdb'}(\Query)]$, there exist
  completed tuples
  $\tf \subseteq \pdb' \setminus \pdb$ with $|\tf| \le |\Phi|$ such that
  \begin{align*}
    \forall ~\langle t:p\rangle \in \tf&:~~ p \in [0,\lambda], \text{ and }\\
    \forall ~\langle t:p\rangle \in \left(\pdb' \setminus \pdb\right) \setminus \tf&:~~ p \in \{0,\lambda\}.
  \end{align*}
\end{theorem}
That is, our upper bound is given by a
completion that has at most $|\Phi|$ added tuples with probability not exactly
$0$ or $\lambda$. Intuitively, each MTP constraint contributes a single non-boundary tuple, which can be thought of as the ``leftover'' probability mass once the rest has been assigned in~full.

This insight allows us to treat MTP query evaluation as a combinatorial optimization problem for the rest of this paper. Thus, we only consider the case where achieving the mean tuple probability exactly leaves us with every individual tuple at its boundary. To see that we can do this, we observe that Theorem~\ref{thm:extrema} leaves a single tuple per MTP constraint not necessarily on the boundary. But this tuple can always be forced to be on the boundary by very slightly increasing the mean  $\bar{p}$ of the constraint, as follows.

\begin{corollary}
Suppose we have an OpenPDB $\gf=(\pdb, \lambda)$ with MTP constraints $\Phi$, and a UCQ $\Query$. 
Suppose further that each relation in $\gf$ has at most $1$ constraint in $\Phi$, and that each constraint allows adding open-world probability mass exactly divisible by $\lambda$.  
Then if $\pdb'$ is a $\Phi$-constrained $\lambda$-completion of $\pdb$ with $\kgp(\Query)=[P_\pdb(\Query),P_{\pdb'}(\Query)]$, we have
\begin{equation*}
    \forall \langle t : p \rangle \in \pdb' \setminus \pdb:~~ p \in \{0,\lambda\}.
\end{equation*}
\end{corollary}

Our investigation into the algorithmic properties of MTP query evaluation will be focused on constraining a single relation, subject to a single combinatorial budget constraint. 
 
 \section{Exact MTP Query Evaluation} \label{sec:exact}
 With Section~\ref{sec:mtp} formalizing MTP constraints and showing that computing upper bounds subject to MTP constraints is a combinatorial problem of choosing which $\lambda$-probability tuples to add in the completion, we now investigate exact solutions.
 
 \subsection{An Algorithm for Inversion-Free Queries}
 We begin by describing a class of queries which admits poly-time evaluation subject to an MTP constraint. We first need to define some syntactic properties of queries.
 
\begin{definition}
 Let $\Query$ be a conjunctive query, and let $at(x)$ denote the set of relations containing variable $x$. We say that $\Query$ is \emph{hierarchical} if for any $x,y$, we have either $at(x) \subseteq at(y)$, $at(y) \subseteq at(x)$, or $at(x) \cap at(y) = \emptyset$. 
\end{definition} 

Intuitively, a conjunctive query being hierarchical indicates that it can either be separated into independent parts (the $at(x) \cap at(y) = \emptyset$ case), or there is some variable that appears in every atom. This simple syntactic property is the basis for determining whether query evaluation on a conjunctive query can be done in polynomial time \citep{Dalvi2007TheDO}. We can further expand on this definition in the context of UCQs.

\begin{definition}
 A UCQ $\Query$ is \emph{inversion-free} if each of its conjuncts is hierarchical, and they all share the same hierarchy.\footnote{See \citet{Jha2011KnowledgeCM} for a more detailed definition.} If $\Query$ is not inversion-free, we say that it has an inversion.
\end{definition}

This query class remains tractable under MTP constraints.

\begin{theorem}\label{thm:invpoly}
   For any inversion-free query $\Query$, evaluating the probability $\overline{P}_{\gf}^\Phi(\Query)$ subject to an MTP constraint is in~PTIME.
\end{theorem}

In order to prove Theorem~\ref{thm:invpoly}, we provide a polytime algorithm for MTP query evaluation on inversion-free queries. As with OpenPDBs, our algorithm depends on Algorithm~\ref{alg:LiftR}, the standard lifted inference algorithm for PDBs. Algorithm~\ref{alg:LiftR} proceeds in steps recursively processing $\Query$ to compute query probabilities in polynomial time for safe queries \citep{Dalvi2012TheDO}. Further details of the algorithm including the necessary preprocessing steps and notation can be found in \citet{Dalvi2012TheDO} and \citet{Gribkoff2014UnderstandingTC}.

\renewcommand\algorithmicthen{}
\begin{algorithm}[tb]
\caption{$\bf Lift^R(\Query,\pdb)$, abbreviated by ${\bf L}(\Query)$ \label{alg:LiftR}}

  \begin{algorithmic}[1]
    \Require UCQ $\Query$ , prob.~database $\pdb$ with constants $T$.
    \Ensure The probability $P_{\pdb}(\Query)$
     \State {\textbf{Step 0}}  ~\emph{Base of Recursion}
     \Indent
     \If{\Query is a single ground atom $t$ }
                \If{$\ax{t:p} \in \pdb$} \Return $p$ \textbf{else} \textbf{return} $0$
       \EndIf 
    \EndIf  
    \EndIndent 
    \State {\textbf{Step 1}} ~\emph{Rewriting of Query}
     \Indent    
    \State Convert \Query to conjunction of UCQ: $\Query_{\land}\!\!=\!\Query_1   \land  \cdots  \land \Query_m$
    \EndIndent     
    \State \textbf{Step 2}  ~\emph{Decomposable Conjunction}    
    \Indent    
     \If{$m>1$  and $\Query_{\land}=\Query_1 \land \Query_2$ where $\Query_1 \perp \Query_2$}
         \State  \Return ${\bf L}(\Query_1) \cdot {\bf L}(\Query_2)$
    \EndIf
     \EndIndent     
  \State \textbf{Step 3}  ~\emph{Inclusion-Exclusion}    
    \Indent  
     \If{$m>1$ but $\Query_{\land}$ has no independent $\Query_i$}
         \State \Return $\sum_{s \subseteq [m]}(-1)^{|s|+1} \cdot {\bf L}\left(\bigvee_{i \in s} \Query_i\right)$
          \EndIf
     \EndIndent   
    \State \textbf{Step 4}  ~\emph{Decomposable Disjunction}    
    \Indent    
     \If{$\Query=\Query_1 \vee \Query_2$ where $\Query_1 \perp \Query_2$}
     \State \Return $1- \left(1-{\bf L}(\Query_1)\right) \cdot \left(1-{\bf L}(\Query_2)\right)$
    \EndIf
     \EndIndent    
    \State \textbf{Step 5}  ~\emph{Decomposable Existential Quantifier}         
    \Indent    
     \If{\Query has a \emph{separator variable} $x$}
     \State \Return $1-\prod_{c \in T} \left(1-{\bf L}(\Query[x/c])\right)$ \label{line:separator}
    \EndIf
     \EndIndent   
    \State \textbf{Step 6}  ~\emph{Fail}  (the query is \#P-hard)     
  \end{algorithmic}
\end{algorithm} 

We now present an algorithm for doing exact MTP query evaluation on inversion-free queries. For brevity, we present the binary case; the general case follows similarly and can be found in appendix. Suppose that we have a probabilistic database $\pdb$, a domain $T$ of constants denoted $c$, a query~$\Query$, and an MTP constraint on relation $R(x, y)$ allowing us to add exactly $B$ tuples with probability $\lambda$. Suppose that $\Query$ immediately reaches Step 5 of Algorithm~\ref{alg:LiftR} (other steps will be discussed later), implying that $x$ and $y$ are unique variables in the query. We let $A(c_x,c_y,b)$ denote the upper query probability of $\Query(x/c_x, y/c_y)$ subject to an MTP constraint allowing budget $b$ on the relevant portion of $R$. That is, $A$ tells us the highest probability we can achieve for a partial assignment given a fixed budget. Observe that we can compute all entries of $A$ using a slight modification of Algorithm~\ref{alg:LiftR} where we compute probabilities with and without each added tuple. 
This will take time polynomial in $|T|$.

Next, we impose an ordering $c_1, \dots, c_{|T|}$ on the domain. Then we let $D(j, c_y, b)$ denote the upper query probability of
\begin{equation*}
    \bigvee_{c \in \{c_1, \dots c_j\}} \Query(x/c, y/c_y)
\end{equation*}
\noindent with a budget of $b$ on the relevant portions of $R$. Then $D(|T|, c_y, b)$ considers all possible substitutions in our first index, meaning we have effectively removed a variable. Doing this repeatedly would allow us to perform exact MTP query evaluation. However, $D$ is non-trivial to compute, and cannot be done by simply modifying Algorithm~\ref{alg:LiftR}. Instead, we observe the following recurrence:
\begin{align*}
    & D(j+1, y/c_y, b)  = \\
    & \qquad \max_{k \in \{1,\dots,b\}} 1 - D(j, y/c_y, b-k) \cdot A(x/c_{j+1}, y/c_y, k).
\end{align*}

Intuitively, this recurrence says that since the tuples from each fixed constant are all independent, we do not need to store which budget configuration on the first $j$ constants got us our optimal solution. Thus, when we add the $j+1$th constant, we just need to check each possible value we could assign to our new constant, and see which gives the overall highest probability. This recurrence can be implemented efficiently, yielding a dynamic programming algorithm that runs in time polynomial in the domain size and~budget.

 Finally, we would like to generalize this algorithm beyond assuming that $\Query$ immediately reaches Step 5 of Algorithm~\ref{alg:LiftR}. Looking at other cases, we see that Steps 0 and 1 have no effect on this recurrence, and Steps 2 and 4 correspond to multiplicative factors. For a query that reaches Step 3 (inclusion-exclusion), we need to construct such $A$ and $D$ for each term in the inclusion-exclusion sum, and follow the analogous recurrence. Notice that the modified algorithm would only work in the case where we can always pick a common variable for all sub-queries to do dynamic programming on -- that is, when the query is inversion-free, as was our assumption.


 \subsection{Queries with Inversion} \label{s:intractability}
We now show that allowing for inversions in safe queries can cause MTP query evaluation to become NP-hard. Interestingly, this means that MTP constraints fundamentally change the difficulty landscape of query evaluation.

To show this, we investigate the following UCQ query.
 \begin{align*}
   M_0 = \exists x\exists y\exists z & \left(R(x,y,z)\land U(x) \right)
                \lor \left(R(x,y,z) \land V(y)\right) \\  
         & \lor \left(R(x,y,z) \land W(z) \right)
                \lor  \left(U(x)\land V(y) \right)\\ 
         & \lor \left(U(x)\land W(z) \right)
                \lor \left(V(y)\land W(z)\right)
\end{align*}

A key observation here is that the query $M_0$ is a \emph{safe} UCQ. That is, if we ignore constraints and evaluate it subject to the closed- or open-world semantics, computing the probability of the query would be polynomial in the size of the database. We now show that this is \emph{not} the case for open-world query evaluation subject to a single MTP constraint on~$R$.

\begin{theorem}\label{thm:hard}
Evaluating the upper query probability bound $\overline{P}_{\gf}^\Phi(M_0)$  subject to an MTP constraint $\Phi$ on $R$ is NP-hard.
\end{theorem}

The full proof of Theorem~\ref{thm:hard} can be found in appendix, showing a reduction from the NP-complete 3-dimensional matching problem to computing~$\overline{P}_{\gf}^\Phi(M_0)$ with an MTP constraint on~$R$. It uses the following intuitive correspondence.

\begin{definition}
  Let $X, Y, Z$ be finite disjoint sets representing nodes, and let $T\subseteq X \times Y \times Z$ be the set of available hyperedges. Then $M\subseteq T$ is a \emph{matching} if for any distinct triples $(x_1,y_1,z_1) \in M, (x_2, y_2,z_2) \in M$, we have that $x_1 \ne x_2, y_1 \ne y_2, z_1 \ne z_2$. The \emph{3-dimensional matching} decision problem is to determine for a given $X, Y, Z, T$ and positive integer $k$ if there exists a matching $M$ with $|M| \ge k$.
\end{definition}
 
 The set of available tuples for $R$ will correspond to all edges in $T$. Our MTP constraint forces a decision on which subset of $T$ to take and include in the $\lambda$-completion. 
 
 However, if we simply queried to maximize $P(R(x,y,z))$, this completion need not correspond to a matching. Instead, we have the conjunct $R(x,y,z)\land U(x)$ which is maximized when each tuple chosen from $R$ has a different $x$ value. Similar conjuncts for $y$ and $z$ ensure that the query is maximized when using distinct $y$ and $z$ values. Putting all of these together ensures that the query probability is maximized when the subset of tuples chosen to complete $R$ form a matching.

 Finally, the last part of the query $(U(x) \land V(y)) \lor (U(x) \land W(z)) \lor (V(y) \land W(z))$ ensures that inference on $M_0$ is tractable, but it is unaffected by the choice of tuples in $R$.

\section{Approximate MTP Query Evaluation} \label{sec:approx}
With Section~\ref{s:intractability} answering definitively that a general-purpose algorithm for evaluating MTP query bounds is unlikely to exist, even when restricted to safe queries, an approximation is the logical next step. We now restrict our discussion to situations where we constrain a single relation, and dig deeper into the properties of MTP constraints to show their submodular structure. We then exploit this property to achieve efficient bounds with guarantees.

\subsection{On the Submodularity of Adding Tuples} \label{subsec:submod}
To formally define and prove the submodular structure of the problem, we analyze query evaluation as a set function on adding tuples. 
We begin with a few relevant definitions.

\begin{definition}
\label{def:setfun}
   Suppose that we have an OpenPDB $\gf$, with an MTP constraint $\varphi$  on a single relation
   $R$, and we let $\otups$ be the set of possible tuples we can add to $R$.
   Then the \emph{set query probability} function $S_{\pdb, \Query}:
   2^{\otups} \to [0,1]$ is defined as
   \begin{equation*}
     S_{\pdb,\Query}(X) = P_{\pdb \cup \{\langle t:\lambda \rangle | t \in X\}}(\Query).
   \end{equation*}
   
 \end{definition}

Intuitively, this function describes the probability of the query as a function of which
open tuples have been added.
It provides a way to reason about the combinatorial properties of this optimization problem.  
Observe that $S_{\pdb,\Query}(\emptyset)$ is the closed-world probability of the query, while $S_{\pdb,\Query}(\otups)$ is the open-world probability.

We want to show that $S_{\pdb,\Query}$ is a submodular set function.
\begin{definition}
A \emph{submodular set function} is a function $f: 2^\Omega \to \mathbb{R}$
such that for every $X\subseteq Y \subseteq \Omega$, and every $x
\in \Omega \setminus Y$, we have that $f(X\cup \{x\}) - f(X) \ge f(Y \cup \{x\})
- f(Y)$.  
\end{definition}

\begin{theorem}
\label{thm:submod}
The set query probability function $S_{\pdb, \Query}$ is submodular for any tuple
independent probabilistic database $\pdb$ and UCQ query $\Query$ without self-joins.
\end{theorem}

This gives us the desired submodularity property, which we can exploit to build efficient approximation algorithms.
 
\subsection{From Submodularity to Approximation}\label{subsec:approxalg}
Given the knowledge that the probability of a safe query without self-joins is submodular in the completion of a single relation, we are now tasked with using this to construct an efficient approximation.
Since we further know the probability is also monotone as we have restricted our language to UCQs, \citet{Nemhauser1978} tells us that we
can get a $1-\frac{1}{e}$ approximation using a simple greedy algorithm. The final requirement to achieve the approximation described in \citet{Nemhauser1978} is that our set function must have the property that $f(\emptyset)=0$. This can be achieved in a straightforward manner as follows.

\begin{definition}
   In the context of the set query probability function of Definition~\ref{def:setfun}, the \emph{normalized} set query probability function $S'_{\pdb, \Query}:
   2^{\otups} \to [0,1]$ is defined as
   \begin{equation*}
     S'_{\pdb,\Query}(X) = P_{\pdb \cup \{\langle t:\lambda \rangle | t \in X\}}(\Query) - P_{\pdb}(\Query).
   \end{equation*}
 \end{definition}
 
 \begin{proposition}
 Any normalized set query probability function $S'_{\pdb,\Query}$ is monotone, submodular, and satifies $S'_{\pdb,\Query}(\emptyset) = 0$.
 \end{proposition}

By simply normalizing the set query probability function, we can now directly apply the greedy approximation described in \citet{Nemhauser1978}. We slightly modify Algorithm~\ref{alg:LiftR} to efficiently compute the next best tuple to add based on the current database, and add it. This is repeated until adding another tuple would violate the MTP constraint. Finally, we say that $P_{\mathit{Greedy}}(\Query)$ is the approximation given by this greedy algorithm and recall that the true upper bound is $\overline{P}_{\gf}^\Phi(\Query)$. We observe that $P_{\mathit{Greedy}}(\Query) \le \overline{P}_{\gf}^\Phi(\Query)$. Furthermore, \citet{Nemhauser1978} tells us the following:
\begin{equation*}
    P_{\mathit{Greedy}}(\Query) - P_\pdb(\Query) \ge (1-\frac{1}{e})(\overline{P}_{\gf}^\Phi(\Query) - P_\pdb(\Query))   
\end{equation*}

Combining these and multiplying through gives us the following upper and lower bound on the desired probability.
\begin{equation*}
    P_{\mathit{Greedy}}(\Query) \le \overline{P}_{\gf}^\Phi(\Query) \le \frac{e\cdot P_{\mathit{Greedy}}(\Query) - P_\pdb(\Query)}{e-1}
\end{equation*}

It should be noted that depending on the query and database, it is possible for this upper bound to exceed $1$.

\section{Discussion, Future \& Related Work} \label{sec:disc}

We propose the novel problem of constraining open-world probabilistic databases at the schema level, without having any additional ground information over individuals. We introduced a formal mechanism for doing this, by limiting the \emph{mean tuple probability} allowed in any given completion, and then sought to compute bounds subject to these constraints. We now discuss remaining open problems and related work.

Section~\ref{sec:exact} showed that there exists a query that is NP-hard to compute exactly, and also presented a tractable algorithm for a class of inversion-free queries. The question remains how hard the other queries are - in particular, is the algorithm presented complete. Is there a complexity dichotomy, that is, a set of syntactic properties that determine the hardness of a query subject to MTP constraints.
Questions of this form are a central object of study in probabilistic databases. 
It has been explored for conjunctive queries \citep{Dalvi2007TheDO}, UCQs \citep{Dalvi2012TheDO}, and a more general class of queries with negation \citep{Fink2016DichotomiesFQ}.

The central goal of our work is to find stronger semantics based on OpenPDBs, while still maintaining their desirable tractability. This notion of achieving a powerful semantics while maintaining tractability is a common topic of study. \citet{Raedt2015ProbabilisticP} study this problem by using a probabilistic interpretation of logic programs to define a model, leading to powerful semantics but a more limited scope of tractability \citep{Fierens2015InferenceAL}. The description logics \citep{Nardi2003TheDL} is a knowledge representation formalism that can be used as the basis for a semantics. This is implemented in a probabilistic setting in, for example, probabilistic ontologies \citep{Riguzzi2012EpistemicAS,Riguzzi2015ReasoningWP}, probabilistic description logics \citep{Heinsohn1994ProbabilisticDL}, probabilistic description logic programs \citep{Lukasiewicz2005ProbabilisticDL}, or the bayesian description logics \citep{10.1007/978-3-319-08587-6_37}.

Probabilistic databases in particular are of interest due to their simplicity and practicality. Foundational work defines a few types of probabilistic semantics, and provides efficient algorithms as well as when they can be applied \citep{Dalvi2004EfficientQE, Dalvi2007TheDO, Dalvi2012TheDO}. These algorithms along with practical improvements are implemented as industrial level systems such as MystiQ \citep{R2008ManagingPD}, SPROUT \citep{Olteanu2009SPROUTLV}, MayBMS \citep{Huang2009MayBMSAP}, and Trio which implements the closely related Uncertainty-Lineage Databases \citep{Benjelloun2007DatabasesWU}. Problems outside of simple query evaluation are also points of interest for PDBs, for example the most probable database \citep{gribkoff2014most}, or of ranking the top-k results \citep{R2007EfficientTQ}. In the context of OpenPDBs in particular, \citet{Grohe2018ProbabilisticDW} study the notion of an infinite open world, using techniques from analysis to explore when this is feasible.

\newpage

\section*{Acknowledgements}
The authors thank Yitao Liang and YooJung Choi for helpful feedback. This work is partially supported by NSF grants \#IIS1657613, \#IIS-1633857, \#CCF-1837129, DARPA XAI grant
\#N66001-17-2-4032, NEC Research and a gift from Intel.

\bibliographystyle{named}
\bibliography{opref}

\newpage

\appendix

\section{Proofs of Theorems, Lemmas, and Propositions}
\subsection{Proof of Proposition~\ref{thm:credal}}
\begin{proof}
  To prove this, we need to show that $\kgp$ is both closed and convex.
  
  Due to the way our constraints are defined, we know that $\kgp=K_\gf
  \cap K^\Phi$, where $K^\Phi$ is the set of all completions satisfying $\Phi$ (but
  not necessarily having all tuple probabilities $\le \lambda$). We already know
  that $K_\gf$ is credal, and thus closed and convex. $K^\Phi$ is a half-space,
  which we also know is closed and convex. The intersection of closed spaces is closed, and the intersection of convex spaces is convex, so $\kgp$ is credal.
\end{proof}

\subsection{Proof of Theorem~\ref{thm:extrema}}

\begin{proof}
  Since $\kgp$ is credal, we are interested here in determining the point of
  extrema of $\kgp$, as this will tell us precisely which completions can
  represent boundaries.

  Consider the construction of the set $\kgp$, and suppose that there are $d$ possible open-world tuples, meaning that $\kgp \subseteq \mathbb{R}^d$. As we observed in the proof of Theorem~\ref{thm:credal}, $\kgp = K_\gf \cap K^{\Phi}$, where $K^\Phi$ is the set of all completions satisfying $\Phi$. We now make three key observations about these sets:
  \begin{enumerate}
      \item Each individual possible open-world tuple is described by the intersection of 2 half-spaces: that is, the tuple on dimension $i$ is described by $x_i \ge 0$ and $x_i \le \lambda$. $K_\gf$ is the intersection of all $2d$ of these half-spaces.
      \item For any individual open-world tuple, the boundaries of the two half-spaces that describe it cannot intersect each other.
      \item An MTP constraint is a linear constraint, and thus can be described by a single half-space. So $K^{\Phi}$ is described by the intersection of these $|\Phi|$ half-spaces.
  \end{enumerate}
  
  Observations 1 and 3, together with Lemma~\ref{lem:hfspext} tells us that any point of extrema of $\kgp$ must be given by the intersection of the boundaries of at least $d$ of the half-spaces that form $\kgp$. Observation 3 tells us that at most $|\Phi|$ of these half-spaces come from MTP constraints, leaving the boundaries of at least $d-|\Phi|$ half-spaces which come from $K_\gf$. Finally, observation 2 tells us that each of these $d-|\Phi|$ half-spaces is describing a different open world tuple. But this means we must have at least $d-|\Phi|$ tuples which lie on the boundary of one of their defining half-spaces: they must be either $0$ or $\lambda$. 
  
\end{proof}

\subsection{Proof of Theorem~\ref{thm:hard}}
Before we present the formal proof, we state and prove 2 Lemmas we will need.
 
 \begin{lemma}
 \label{lem:ordercomp}
    Suppose we have two completions $P_1$ and $P_2$ of $R$, which only differ on a single triple, that is $P_1 = P_0 \cup \{x_1, y_1, z_1\}$ and $P_2 = P_0 \cup \{x_2, y_2, z_2\}$. Further suppose that $y_1=y_2$, $z_1=z_2$, and that $P_0$ contains no triples with x-value $x_1$, but does contain at least 1 triple with x-value $x_2$. Then $P_1(M_0) > P_2(M_0)$.
 \end{lemma}
 \begin{proof}
   We will apply a similar technique here to the one used to prove Theorem~\ref{thm:submod}, where we directly examine $\Delta$, the logical formula found by grounding $M_0$. Since $M_0$ is a union of conjunctive queries, $\Delta$ must be a DNF. Each conjunct either does not contain $R$, in which case it does not vary with the choice of completion, or it contains it exactly once. Any conjunct containing an atom of $R$ not assigned probability by a completion is logically false.
   
   In order to prove that $P_1(M_0) > P_2(M_0)$, let us compare the ground atoms that result from each. It is clear that the only spot on which they differ is on conjuncts involving $R(x_1,y_1,z_1)$ or $R(x_2,y_1,z_1)$. Any conjuncts involving one of these and $V$ or $W$ will also have an identical effect on the probability of the query, since the completions are identical over $y$ and $z$.
   
   Finally this means we need to compare the term $R(x_1,y_1,z_1),U(x_1)$ with the term $R(x_2,y_1,z_1),U(x_2)$. Observe that we know $P_0$ contains triples with x-value $x_2$, which means the term only contributes new probability mass when $U(x_2)$ is true \emph{and} none of the other triples involving $x_2$ are true. However, $P_0$ does not contain any triples with x-value $x_1$, so the term $R(x_1,y_1,z_1),U(x_1)$ contributes the maximum probability possible. Thus, for any choice of probabilities on $U$ such that $U(x_2)<1$, we have that $P_1(M_0) > P_2(M_0)$.
 \end{proof}
 
 \begin{lemma}
 \label{lem:maxmatch}
 The upper bound $\overline{P}(M_0)$ is maximized if and only if $P$ is a completion formed by a matching of size $k$, where $k$ is the maximum number of tuples with probability $\lambda$ that can be added to $R$ in $M_0$.
 \end{lemma}
  \begin{proof}
   Observe that if we begin with a completion given by a matching, we can repeatedly apply Lemma~\ref{lem:ordercomp} to arrive at any completion. Thus a completion given by a matching must have higher probability than any completion not given by a matching.
 \end{proof}

Finally, we are ready to present the proof of Theorem~\ref{thm:hard}.
 
 \begin{proof}
 Suppose we are given an instance of a 3-dimensional matching problem $X,Y,Z,T$ and an integer $k$. Let $U(x), V(y), W(z)$ be $0.8$ wherever $x \in X, y \in Y$, or $z \in Z$ respectively, and 0 everywhere else. Additionally, let $R(x,y,z)$ be unknown for any $(x,y,z) \in T$, and 0 otherwise. Finally, we place an MTP constraint on $R$ ensuring that at most $k$ tuples can be added, and let $\lambda = 0.8$. Then Lemma~\ref{lem:maxmatch} tells us that $M_0$ evaluated on this database will be maximized if and only if the completion used corresponds to a matching of size $k$. We determine this probability $P_{max}$ using a standard probabilistic database query algorithm, and fixing $R$ to have entries $0.8$ for some disjoint set of triples.
 
 Finally, we use our oracle for MTP constrained query evaluation to check $\overline{P}(M_0)$ with the database we constructed from the matching problem. We compare the upper bound given by the oracle, and if it is equal to $P_{max}$, Lemma~\ref{lem:maxmatch} tells us that a matching of size $k$ does exist. Similarly, if the upper bound given by the oracle is lower than $P_{max}$, Lemma~\ref{lem:maxmatch} tells us a matching of size $k$ does not exist.
 \end{proof}

\begin{lemma}
  \label{lem:hfspext}
  If $S \subseteq \mathbb{R}^n$ is a set formed by the intersection of $k<n$
  half-spaces, $S$ has no points of extrema. 
\end{lemma}

\begin{proof}
  Written as a set of linear equalities, the solution clearly must have at least
  1 degree of freedom. This indicates that for any potential extrema point $x$,
  one can move in either direction along this degree of freedom to construct an
  open line intersecting $x$, but entirely contained in $S$.
\end{proof}

\subsection{Proof of Theorem~\ref{thm:submod}}
\begin{proof}
Without directly computing probabilities, let us inspect $\Delta$, the logical
formula we get by grounding $\Query$. $\Query$ is a union of conjunctive
queries, and thus $\Delta$ is a very large disjunction of conjuncts. Each conjunct can
contain our constrained relation $R$ at most once due to the query not having self-joins, and any one of these conjuncts
containing an atom of $R$ not assigned any probability is logically false.

Next, to show that $S_{\pdb, \Query}$ is submodular, let $X \subseteq Y
\subseteq \otups$, and let $x \in \otups \setminus Y$ be given. We assign names
to the following subformulas of $\Delta$
\begin{itemize}
\item[--] $\alpha$ ($\beta$) is the disjunction of all conjuncts of $\Delta$ which are not logically false due to missing $R_B$ tuples
  in $X$ ($Y$)
\item[--] $\gamma$ is the disjunction of all conjuncts of $\Delta$ containing
  the tuple $x$
\end{itemize}

Additionally, since $X \subseteq Y$, we
also know that $\alpha \Rightarrow \beta$.  Now, we make a few observations relating
these quantities with our desired values for submodularity:
\begin{itemize}
\item[--] $S_{\pdb, \Query}(X) = P(\alpha)$
\item[--] $S_{\pdb, \Query}(Y) = P(\beta)$
\item[--] $S_{\pdb, \Query}(X \cup \{x\}) = P(\alpha \lor \gamma)$
\item[--] $S_{\pdb, \Query}(Y \cup \{x\}) = P(\beta \lor \gamma)$
\end{itemize}

Finally, we have the following:
\begin{align*}
  S_{\pdb, \Query}(X \cup \{x\}) - S_{\pdb, \Query}(X) &= P(\alpha \lor \gamma) - P(\alpha) \\
                                                       &= P(\neg \alpha \land \gamma) \\
                                                       & \ge P(\neg \beta \land \gamma) \\
                                                       &= P(\beta \lor \gamma) - P(\beta) \\
                                                       &= S_{\pdb, \Query}(Y \cup \{x\}) - S_{\pdb, \Query}(Y)
\end{align*}

 \end{proof}

\section{General Algorithm for Inversion-Free Queries}
We now present an algorithm for doing exact MTP query evaluation on inversion-free queries. Suppose that we have a probabilistic database $\pdb$, a domain $T$ of constants denoted $c$, a query $\Query$, and an MTP constraint on relation $R(x_1, x_2, \dots, x_r)$ allowing us to add $B$ tuples. For any $I \subseteq \{1,\dots,r\}$, we let $A(x_{i_1}/c_{i_1}, x_{i_2}/c_{i_2},\dots, x_{i_{|I|}}/c_{i_{|I|}}, b)$ denote the upper query probability of $\Query(x_{i_1}/c_{i_1}, x_{i_2}/c_{i_2},\dots, x_{i_{|I|}}/c_{i_{|I|}})$ subject to an MTP constraint allowing budget $b$ on the relevant portion of $R$. That is, $A$ tells us the highest probability we can achieve for a partial assignment given a fixed budget. Observe that we can compute all entries of $A$ using a slight modification of Algorithm~\ref{alg:LiftR}. This will take time polynomial in $|T|$.

Next, we impose an ordering $c_1, c_2, \dots, c_{|T|}$ on the domain. For any $I \subseteq \{1,\dots,r\}$, we let $D(j, x_{i_2}/c_{i_2},\dots, x_{i_{|I|}}/c_{i_{|I|}}, b)$ denote the upper query probability of
\begin{equation}
    \bigvee_{c \in \{c_1, \dots c_j\}} \Query(x_{i_1}/c, x_{i_2}/c_{i_2},\dots, x_{i_{|I|}}/c_{i_{|I|}})
\end{equation}

with a budget of $b$ on the relevant portions of $R$. Then $D(|T|, x_{i_2}/c_{i_2},\dots, x_{i_{|I|}}/c_{i_{|I|}}, b)$ considers all possible substitutions in our first index, meaning we no longer need to worry about it. Doing this repeatedly would allow us to perform exact MTP query evaluation. However, $D$ is non-trivial to compute, and cannot be done by simply modifying Algorithm~\ref{alg:LiftR}. Instead, we observe the following recurrence:
\begin{align*}
    D&(j+1, x_{i_2}/c_{i_2},\dots, x_{i_{|I|}}/c_{i_{|I|}}, b) \\  &= 
    \max_{k \in \{1,\dots,b\}} 1 - D(j, x_{i_2}/c_{i_2},\dots, x_{i_{|I|}}/c_{i_{|I|}}, b-k) \\ \cdot & A(x_{i_1}/c_{j+1}, x_{i_2}/c_{i_2},\dots, x_{i_{|I|}}/c_{i_{|I|}}, k)
\end{align*}

Intuitively, this recurrence is saying that since the tuples from each fixed constant are independent of each other, we can add a new constant to our vocabulary by simply considering all combinations of budget assignments. This recurrence can be implemented efficiently, yielding a dynamic programming algorithm that runs in time polynomial in the domain size and budget.

The keen reader will now observe that the above definition and recurrence only make sense if $\Query$ immediately reaches Step 5 of Algorithm~\ref{alg:LiftR}. While this is true, we see that Steps 0 and 1 have no effect on this recurrence, and Steps 2 and 4 correspond to multiplicative factors. For a query that reaches Step 3: inclusion-exclusion, we indeed need to construct such matrices for each sub-query. Notice that the modified algorithm would only work in the case where we can always pick a common $x_i$ for all sub-queries to do dynamic programming on - that is, when the query is inversion-free.

\end{document}